\documentclass[twoside,11pt]{article}

\usepackage{bbm}
\usepackage{enumerate}
\usepackage{mathrsfs}
\usepackage{mathtools}
\usepackage{microtype}
\usepackage{amssymb}
\usepackage{amsmath}
\usepackage{xcolor}
\usepackage{bigints}

\usepackage{comment}

\usepackage{jmlr2e}

\hypersetup{ hidelinks }
\allowdisplaybreaks

\usepackage{lastpage}
\jmlrheading{24}{2024}{1-\pageref{LastPage}}{01/23}{}{}{Park, Song, Kang}

\ShortHeadings{Variable Splitting}{Yesom Park, Changhoon Song, and Myungjoo Kang}
\firstpageno{1}

\newcommand{\cB}{\mathcal{B}}

\newcommand{\cD}{\mathcal{D}}

\newcommand{\cI}{\mathcal{I}}
\newcommand{\cJ}{\mathcal{J}}

\newcommand{\cL}{\mathcal{L}}

\newcommand{\cN}{\mathcal{N}}

\newcommand{\bN}{\mathbb{N}}

\newcommand{\bR}{\mathbb{R}}

\newcommand{\eqqcolon}{=:}
\newcommand*\diff{\mathop{}\!\mathrm{d}}

\definecolor{forestgreen}{RGB}{34,159,44}
\definecolor{deepred}{RGB}{235,0,20} 
\definecolor{custompink}{RGB}{255,123,123}

\newtheorem{assumption}{Assumption}
\newtheorem{criterion}{Criterion}
\newtheorem{premise}{Premise on residual minimization}

\begin{document}

\title{Beyond Derivative Pathology of PINNs: Variable Splitting Strategy with Convergence Analysis}

\author{\name Yesom Park\thanks{Equal contribution authors} \email yeisom@snu.ac.kr\\
        \addr Department of Mathematical Science\\
        Seoul National University 
        \AND
        \name Changhoon Song\footnotemark[1] \email changhoon.song93@snu.ac.kr \\
        \addr Department of Mathematical Science\\
        Seoul National University
        \AND
        \name Myungjoo Kang\thanks{Correspondence to: \tt <mkang@snu.ac.kr>.} \email mkang@snu.ac.kr\\
        \addr Department of Mathematical Sciences\\
        Seoul National University
        }

\editor{Jean-Philippe Vert}

\maketitle

\begin{abstract}%
Physics-informed neural networks (PINNs) have recently emerged as effective methods for solving partial differential equations (PDEs) in various problems.
Substantial research focuses on the failure modes of PINNs due to their frequent inaccuracies in predictions.
However, most are based on the premise that minimizing the loss function to zero causes the network to converge to a solution of the governing PDE. 
In this study, we prove that PINNs encounter a fundamental issue that the premise is invalid. 
We also reveal that this issue stems from the inability to regulate the behavior of the derivatives of the predicted solution.
Inspired by the \textit{derivative pathology} of PINNs, we propose a \textit{variable splitting} strategy that addresses this issue by parameterizing the gradient of the solution as an auxiliary variable. 
We demonstrate that using the auxiliary variable eludes derivative pathology by enabling direct monitoring and regulation of the gradient of the predicted solution.
Moreover, we prove that the proposed method guarantees convergence to a generalized solution for second-order linear PDEs, indicating its applicability to various problems.
\end{abstract}


\section{Introduction}\label{sec:intro}

Partial differential equations (PDEs) that explain the rates of change and interactions of varying quantities have become indispensable mathematical descriptions of real-world phenomena, including the laws of physics \citep{cannon2012evolution, temam2001navier, ikawa2000hyperbolic, cazenave1989introduction, collins2019parabolic}, principles of engineering \citep{sunden2016heat, kulov2014mathematical}, medical science \citep{jyothi2017application, joshi2002optimal}, and finance \citep{black2019pricing, forsyth2007numerical}. 
Because many PDEs do not have a sufficiently smooth classical solution, generalized solutions whose derivatives may not necessarily exist are considered to describe various problems.
Accordingly, obtaining generalized solutions for PDEs is an important mathematical and practical endeavor.
However, since it is generally not feasible to analytically obtain solutions for PDEs, numerical techniques are commonly used to approximate the solution.
Given the broad applicability of generalized solutions, it is important to develop numerical methods that guarantee convergence to the solution.

With the rapid development of deep-learning techniques, interest in developing deep-learning approaches for solving PDEs has increased \citep{raissi2018deep, yu2018deep, li2020fourier}.
One notable method is the physics-informed neural network (PINNs) \citep{raissi2019physics}, which incorporate physical information into the training of neural networks by reviving the original idea of \cite{lee1990neural} using the modern explosion of computational resources.
Building upon the expressive power of networks, as demonstrated by the universal approximation theory \citep{hornik1989multilayer, leshno1993multilayer, li1996simultaneous}, PINNs use neural networks to represent PDE solutions.
By aggregating the residual form of the governing PDE and its boundary conditions as soft penalties, they defined a physics-integrated loss function to measure the network compliance with the PDE and boundary conditions.
The network parameters are then adjusted successively to minimize loss.
Their straightforward implementation and direct applicability to any arbitrary PDE without requiring additional attention make them useful for a wide range of applications.

Despite their significant practical potential, PINNs fail to produce accurate predictions in several cases \citep{fuks2020limitations, krishnapriyan2021characterizing}.
Several studies have explained why PINNs sometimes fail to attain the correct solution, most of which are associated with training issues. 
For example, PINNs could become stuck in a spurious local minimum and fail to reach a global minimum because of complicated loss functions with higher-order derivatives \citep{wang2021understanding, basir2022critical}, or discrepancies between the different components of the losses \citep{wang2021understanding, wang2022and}.
Failure to capture the effective solution region due to improper selection of training collocation points has also been proposed \citep{daw2023mitigating, wu2023comprehensive, gao2023failure}.
However, the mechanism underlying the PINN approach is not fully understood.

This study raises a fundamental question that must be answered before addressing training issues:
\begin{center}
    \textit{Is it always guaranteed that a predicted solution with a small PINN loss corresponds to a good approximation of a solution of the governing PDE?}
\end{center} 
We show the incompleteness of PINNs in that this question does not hold.
We first demonstrate that a convergent sequence of networks whose PINN loss is optimally minimized to zero may deviate from a solution of the PDE in Theorem \ref{thm:pinn_failure}. 
In other words, even if the training issues raised earlier are successfully addressed, PINNs can fail to attain a PDE solution.
Because the reliability of numerical methods for solving PDEs is crucial for their application in real-world situations, non-convergence is a lethal disadvantage of PINNs. 

Furthermore, we propose a novel perspective on the failure modes of PINNs by showing that the uncontrolled behavior of the derivatives of the predicted solution contributes to the non-convergence of PINNs.
In particular, Remark \ref{rem:gradient_pinn} shows that the struction of PINNs cannot regulate the gradient of the predicted solution or prevent it from blowing up.
We call the failure to modulate the derivatives of the predicted solution by \textit{derivative pathology}.
The derivative pathology of PINNs
is a new perspective that has been overlooked and ultimately precludes the convergence of PINNs.
This indicates that both the function values and the aspects of derivatives matter in approximating a solution of PDEs through PINNs.

Drawing inspiration from our analysis of the fundamental failure modes of PINNs, we propose a \textit{variable splitting} strategy to address the defects of PINNs and demonstrate their convergence into generalized solutions for second-order linear PDEs. 
The main problem of PINNs lies in the uncontrolled gradient behavior. 
In this regard, in addition to the primary variable that approximates the solution, variable splitting parameterizes the gradient of the solution as an auxiliary variable, allowing us to directly oversee and regulate it.

Subsequently, the relationship between the primary and auxiliary variables is constrained by soft penalties in  accordance with the governing PDE.
We demonstrate that the proposed variable splitting strategy eliminates the flaws of PINNs by ensuring the convergence of both the primary variable and its gradient through the convergence of the auxiliary variable in Theorem \ref{thm:u_convergence_from_V}.
Moreover, we prove that this leads to the convergence of variable splitting to a generalized solution of the second-order linear PDEs in Theorem \ref{thm:conv_weak_2ndlinear}.
We further discuss the practical advantages and effects of variable splitting on PINNs.

The rest of this paper is organized as follows.
In the remainder of the introduction, we provide a brief survey of relevant literature and some related approaches. 
In Section \ref{sec:preliminaries}, we introduce notation, mathematical setup, and preliminaries.
Upon overview of PINNs in Section \ref{sec:preliminaries}, we present the fundamental failure mode of PINNs and its cause in Section \ref{sec:pinn_failure}.
In Section \ref{sec:variable_splitting}, we present the variable splitting strategy and prove its convergence to a generalized solution of second-order linear PDEs in Section \ref{sec:variable_splitting}.
We discuss other strengths and implications of the proposed method in the same section.
Finally, we conclude the paper in Section \ref{sec:conclusion}.

\subsection{Related Works}
\begin{itemize}
    \item \textbf{Characterization of failure modes of PINNs.}
    Despite the remarkable practical demand for various applications, the accurate training of PINNs remains a challenge.
    This has been reflected in several studies aimed at identifying the failure modes of PINNs.
    
    Training issues have been postulated to be the primary cause of pathological behaviors.
    Imbalances between the different components of the physics-embedded loss of PINNs have been identified as contributing factors to optimization issue \citep{wang2021understanding, wang2022and}.
    Learning rate scheduling \citep{wang2021understanding}, adaptive training techniques \citep{wang2022and}, and selection of appropriate hyperparameters \citep{sun2020surrogate, bischof2021multi} have been suggested as remedies to address this issue.  
    Some studies have raised concerns regarding the complexity of the loss landscape for higher-order PDEs \citep{wang2021understanding, basir2022critical} or for those with large coefficients \citep{krishnapriyan2021characterizing}.
    \cite{wong2022learning} explored the tendency of PINNs to become ensnared in deceptive local minima due to poor initializations and proposed training in sinusoidal spaces.
    Other recent developments \citep{daw2023mitigating, wu2023comprehensive, nabian2021efficient} concentrated on the impact of collocation point selection on the training of PINNs and presented new sampling schemes.
    
    Most closely related to this study, \cite{wang20222} demonstrated for the Hamilton-Jacobi-Bellman equations that even if the $L^p$ PINN loss approaches zero, the predicted solution may not be the actual solution unless $p$ is sufficiently large.
    They proposed $L^\infty$ loss function, which requires solving a min-max problem and presents challenges in achieving stable training.
    In contrast, we address different classes of PDEs and minimize the $L^p$ loss function without solving min-max problems.
    
    \item \textbf{Convergence analysis of PINNs.}
    With the popularity of PINNs for approximating PDEs across diverse scientific disciplines, a rigorous convergence analysis has been developed to fully understand their capabilities and limitations.
    \cite{shin2020convergence} showed that the convergence of the $L^\infty$ PINN loss implies the uniform convergence of the network to the classical solution for second-order linear elliptic PDEs with Dirichlet boundary condition.
    \cite{son2021sobolev} showed the convergence of the network to the classical solution for homogeneous Burgers' and Fokker-Planck equations, provided by minimizing the conventional $L^2$ PINN loss to zero.
    Moreover, they also showed that by using the $H^1$ PINN loss function, it is possible to achieve convergence in the Sobolev norm for solutions of these PDEs.
    On the other hand, we provide a convergence analysis for generalized solutions and hence cover a wider range of solutions than in the aforementioned works.
    
    As the PINN loss is approximated using discrete collocation points in practice, another line of research has theoretically analyzed the generalization error.
    However, it has been demonstrated that minimizing the discrete loss may not result in convergence in the continuous norm \citep{shin2020convergence, shin2023error}. 
    To address this issue, \cite{shin2020convergence} and \cite{wu2022convergence} have suggested modifications of the PINN loss using the H\"older norm of the neural network and prove its convergence, which is unfortunately intractable in practice.
    \cite{de2022error} established analysis of generalization error of PINNs for solving linear Kolmogorov PDEs.
    Since this study aims to address continuous norms, the generalization error is not a concern.
    
    \item \textbf{Variable augmentation.}
    The separation of variables is a classical method for solving differential equations that allows the equation to be rewritten as a system of simpler equations \citep{bo1990least}.
    Typically, higher-order derivatives are augmented as additional variables that decouple the governing equation into a set of lower-order differential equations that are comparatively easy to solve.
    Recent attempts have been made to introduce additional variables into the PINN approach.
    \cite{Basir_2023} parameterized the vorticity as an additional output of the network for solving Stokes' equation.
    \cite{park2023p} utilized the gradient of the $p$-Poisson equation as an auxiliary variable. 
    The auxiliary variable was introduced to reduce the second-order derivative to first-order derivatives to address the computational inefficiency and challenge of the landscape of the PINN loss, where the second-order derivative is embedded.
    Moreover, by incorporating a gradient as an additional variable, \cite{park2023resdf} separated the PINN optimization problem into subproblems to minimize the intricate objective function in a relatively succinct manner.
    While these prior works introduced auxiliary variables for solving differential equations, the auxiliary variables serve a purpose distinct from that of this study, as they were utilized to facilitate optimization and computational cost.
\end{itemize}

\section{Preliminaries}\label{sec:preliminaries}
This section introduces the notations, definitions, and prerequisite knowledge used in this study.

\subsection{Notation}
In this study, $\Omega$ is a $d$ dimensional domain for $d\in\bN$.
Typically, $\partial\Omega$ and $\bar{\Omega}$ refer to the boundary and closure of $\Omega$, respectively.
The volumes of $\Omega$ and $\partial\Omega$ are denoted as $\left\vert\Omega\right\vert$ and $\left\vert\partial\Omega\right\vert$, respectively.

Each $i$ and $j$ is an index of a matrix or vector; for example, $A=\left(a_{ij}\left(x\right)\right)\in\bR^{d\times d}$ and $b=\left(b_i\left(x\right)\right)\in\bR^d$, respectively.
If all the components of a matrix or vector are equal, we omit indices.
For vector $v\in\bR^d$, $\left\vert v\right\vert$ denotes the Euclidean vector norm.

To cover a wider range of solutions than classical solutions, we use weak differential operator $D$.
With partial differential operators $D_i=\frac{\partial}{\partial x_i}$, we write the differential operator as $D=\left(D_1,\ldots,D_d\right)$, where $D\cdot$ denotes the divergence operator.
We slightly abuse the terminology gradient to refer to both weak and classical derivatives if there is no room for confusion.
Higher-order derivatives borrow a multi-index $\alpha=\left(\alpha_1,\ldots,\alpha_d\right)\in\bN^d_0$ and $D^\alpha=D^{\alpha_1}_1\cdots D^{\alpha_d}_d$, where $\bN_0$ is the set of non-negative integers.
The size of the index $\alpha$ is expressed as $\left\vert\alpha\right\vert=\alpha_1+\cdots+\alpha_d$.

For PDEs of the order $k\in\bN$, we aim to address a generalized solution in the Sobolev space $W^{k-1,p}\left(\Omega,\bR\right)$ with $p\in\left[1,\infty\right]$ that satisfies PDE in a weak sense.
Following the conventional definition of the Sobolev space $W^{k,p}\left(\Omega,\bR\right)$, the vector-valued Sobolev space $W^{k,p}\left(\Omega,\bR^d\right)$ and its norm are defined as follows:
\begin{align}
    W^{k,p}\left(\Omega,\bR^d\right) &= \left\{ V:\Omega\rightarrow\bR^d \mid D^\alpha V_i\in L^p\left(\Omega,\bR\right), \, \forall \left\vert\alpha\right\vert \le k, 1\le i\le d\right\},\\
    \left\Vert V\right\Vert_{W^{k,p}\left(\Omega,\bR^d\right)} &= \left\{\begin{array}{cl}
        \left( \displaystyle\sum_{i=1}^d \left\Vert V_i\right\Vert_{W^{k,p}\left(\Omega,\bR\right)}^p \right)^{\frac{1}{p}} & p\in\left[1,\infty\right), \\
        \displaystyle\max_{1\le i\le d}\left\Vert V_i\right\Vert_{W^{k,\infty}\left(\Omega,\bR\right)} & p=\infty.
        \end{array}\right.
\end{align}
With the above definitions, the Sobolev norm of $u\in W^{k,p}\left(\Omega,\bR\right)$ can be separated by $L^p$-norm of $u$ and the Sobolev norm of $Du$ as follows:
\begin{align}
    \left\Vert u\right\Vert^p_{W^{k,p}\left(\Omega,\bR\right)} &= \left\Vert u\right\Vert^p_{L^p\left(\Omega,\bR\right)} + \left\Vert Du\right\Vert^p_{W^{k-1,p}\left(\Omega,\bR^d\right)},
    \intertext{for $p\in\left[1,\infty\right)$, and}
    \left\Vert u\right\Vert_{W^{k,\infty}\left(\Omega,\bR\right)} &= \max\left\{ \left\Vert u\right\Vert_{L^\infty\left(\Omega,\bR\right)} , \left\Vert Du\right\Vert_{W^{k-1,\infty}\left(\Omega,\bR^d\right)}\right\}.
\end{align}
Provided that the codomain of a function is clear, we omit the codomain and write the space or norm briefly as $L^p\left(\Omega\right)$ or $\left\Vert\cdot\right\Vert_{L^p\left(\Omega\right)}$.
We also use the trace operator $T:W^{1,p}\left(\Omega,\bR\right)\rightarrow L^p\left(\partial\Omega,\bR\right)$ introduced by \cite{gagliardo1957caratterizzazioni} as a continuous extension of the boundary operator that restricts a function to the boundary $\partial\Omega$ of its domain $\Omega$.
The operator norm of $T$ is denoted as $\left\Vert T\right\Vert$.

\subsection{PDEs and Generalized Solutions}
Focusing on second-order linear PDEs with Dirichlet boundary conditions, we introduce a generalized solution and assumptions for PDEs to have a generalized solution.
We consider a class of PDEs
\begin{equation}\label{eq:pde}
    \left\{
    \begin{alignedat}{3}
        \cN\left[u,Du,D^2u\right]\left(x\right) &= f\left(x\right), &\quad x\in\Omega,\\
        u\left(x\right) &= g\left(x\right), &\quad x\in\partial\Omega,
    \end{alignedat}
    \right.
\end{equation}
where $\cN$ denotes a second-order linear partial differential operator, $f:\Omega\rightarrow\bR$ denotes a source function, and $g:\partial\Omega\rightarrow\bR$ denotes a boundary function.

Being represented in either divergence form,
\begin{equation}\label{eq:div_form}
    \cN\left[u,Du,D^2u\right]\left(x\right) = -\sum_{i,j=1}^d D_j\left(a_{ij}\left(x\right)D_iu\left(x\right)\right) + \sum_{i=1}^d b_i\left(x\right)D_iu\left(x\right) +c\left(x\right) u\left(x\right),
\end{equation}
or non-divergence form,
\begin{equation}\label{eq:nondiv_form}
    \cN\left[u,Du,D^2u\right]\left(x\right) = -\sum_{i,j=1}^d a_{ij}\left(x\right)D^2_{i,j}u\left(x\right) + \sum_{i=1}^d \tilde{b}_i\left(x\right)D_iu\left(x\right) +c\left(x\right) u\left(x\right),
\end{equation}
which are equivalent if the principal coefficients $a_{ij}:\Omega\rightarrow\bR$ are differentiable, we may assume that the operator $\cN$ is of the divergence form \eqref{eq:div_form}.
We use notational shortcuts $A=\left(a_{ij}\right):\Omega \rightarrow \bR^{d\times d}$ and $b=\left(b_i\right):\Omega\rightarrow\bR^d$, to express $\cN \left[u,Du, D^2u\right]$ of \eqref{eq:div_form} in an index-free form:
\begin{equation}
   \cN\left[u,Du,D^2u\right] = -D\cdot\left(ADu\right) + b^\text{T}Du + cu.
\end{equation}

A \textit{classical solution} $u\in C^2\left(\Omega,\bR\right)\cap C\left(\bar{\Omega},\bR\right)$ is a function that satisfies \eqref{eq:pde} for each $x\in\Omega$; however, numerous PDEs have no classical solution.
Rather, to describe a broader phenomenon, weakly differentiable functions $u\in W^{1,p}\left(\Omega,\bR\right)$ that satisfy the PDE in a weak sense are considered. 
\textit{Generalized solutions} are defined as broader notions of the solutions to \eqref{eq:pde}.
\begin{definition}[Generalized Solutions]\label{def:weak_sol} 
    For $p\in\left[1,\infty\right]$, $u\in W^{1,p}\left(\Omega,\bR\right)$ is a generalized solution to \eqref{eq:pde} if it satisfies
    \begin{equation}\label{eq:weak_sol}
        \left\{
        \begin{alignedat}{2}
            \cJ\left(u,\phi\right) &= \int_\Omega f\phi\ \diff x,\\
            Tu &=g,
        \end{alignedat}
        \right.
    \end{equation}
    for all $\phi\in C^\infty_c\left(\Omega,\bR\right)$, where $\cJ$ is a bilinear form defined as 
    \begin{equation}
        \cJ\left(u,\phi\right)= \bigintss_\Omega \left\{\sum_{i,j=1}^d a_{ij} D_iuD_j\phi + \sum_{i=1}^d b_i \left(D_iu\right)\phi + cu\phi\right\} \diff x.
    \end{equation}
\end{definition}

Although every classical solution of the PDE satisfies this equation, owing to integration by parts, the converse does not hold.
This is because the generalized solution requires only first-order weak derivatives.
Therefore, the generalized solution covers a broader range of solutions for PDEs and describes several real-world phenomena.

In order for a generalized solution of PDEs \eqref{eq:pde} to be well-defined, certain assumptions are necessary.

\begin{assumption}\label{assump:domain}
    The domain $\Omega\subset\bR^d$ is bounded with a Lipschitz boundary.
\end{assumption}
\begin{assumption}\label{assump:coeffs}
    Coefficients $a_{ij},b_i,c:\Omega\rightarrow\bR$ are bounded in $\Omega$.
    In addition, $a_{ij}$ has a weak derivative.
\end{assumption}
\begin{assumption}\label{assump:sources} 
    The source function $f$ and boundary function $g$ are in $L^p\left(\Omega,\bR\right)$ and $L^p\left(\partial\Omega,\bR\right)$, respectively.
\end{assumption}
\begin{assumption}\label{assump:sol} 
    PDE \eqref{eq:pde} has a generalized solution satisfying \eqref{eq:weak_sol} in $W^{1,p}\left(\Omega\right)$.
\end{assumption}
Note that Assumptions \ref{assump:domain}--\ref{assump:sources} are necessary for generalized solutions to be well-defined.
By definition, the trace operator induces Assumption \ref{assump:domain} and $g\in L^p\left(\partial\Omega,\bR\right)$ in Assumption 3.
The integrability of the left-hand side of \eqref{eq:weak_sol} imposes the integrability of the coefficients in Assumption \ref{assump:coeffs} and source function $f$ in Assumption \ref{assump:sources}.
The final assumption is essential.
We refer to \cite{dacorogna2007direct} and \cite{evans2022partial} for further details and the general theory of generalized solutions of PDEs.

\subsection{Physics-Informed Neural Networks (PINNs)}
PINNs are among the most prominent deep-learning frameworks for approximating solutions of PDEs.
Drawing on the expressive power of networks as demonstrated by universal approximation theory \citep{hornik1989multilayer, hornik1990universal, li1996simultaneous}, PINNs harness neural networks as an ansatz to approximate the solutions of PDEs.
As a solution that minimizes the residuals of the governing PDE, PINNs compel the network to approximate the solution of the PDE by minimizing the residuals:
\begin{align}
     \cL_{\cN,p}\left(u\right)&=\left\Vert\cN\left[u,Du,D^2u\right]-f\right\Vert_{L^p\left(\Omega\right)},\label{eq:pinn_loss_pde}\\
     \cL_{\cB,p}\left(u\right)&=\left\Vert u\vert_{\partial\Omega}-g \right\Vert_{L^p\left(\partial\Omega\right)}\label{eq:pinn_loss_bdy}.
\end{align}

The \textit{PDE loss} $\cL_{\cN,p}$ penalizes the violations of PDEs in the domain, and the \textit{boundary loss} $\cL_{\cB,p}$ measures the misfit of the boundary condition.
Although some PINNs exploit the data fitting loss, which reduces the gap between the measurements and predictions if available, we neglect the data fitting term to concentrate on the PDE itself.
For similar reasons, we do not consider any other regularization terms.
Consequently, we define the \textit{$L^p$ PINN loss} as a linear combination of the two losses above:
\begin{equation}\label{eq:pinn_loss}
    \cL_p \left(u\right) = \lambda_{\cN}\cL_{\cN,p}\left(u\right)+\lambda_\cB \cL_{\cB,p}\left(u\right),
\end{equation}
where $\lambda_\cN,\lambda_\cB>0$ are the regularization parameters that control the compromise between the two loss terms.

Note that if a classical solution of a PDE exists, it minimizes the PINN loss \eqref{eq:pinn_loss}.
Conversely, a twice-differentiable function $u$ with $\cL_p\left(u\right)=0$ is a classical solution.
Accordingly, PINNs consider the following optimization problem of the PINN loss to solve the PDE:
\begin{equation}\label{eq:pinn}
    \underset{u}{\text{minimize}}\ \cL_p\left(u\right).
\end{equation}
We do not specify the function space of the \textit{decision variable} $u$ in \eqref{eq:pinn}, because the activation function and $p$ used in an actual PINN implementation may alter the function space in which $u$ lies.
Nevertheless, $u$ must be twice differentiable and $L^p$ integrable for the PINN loss to be well-defined.

However, obtaining a minimizer with an exact loss of zero is impossible.
In practice, PINNs numerically find the decision variable with a small loss.
Therefore, PINNs approximate the PDE solution based on the following primitive premise:

\begin{premise}\label{premise:PINN}
    \textit{A predicted solution with a sufficiently small PINN loss would approximate a solution of the governing PDE.}
\end{premise}

Since the generalized solution is not twice differentiable, the PINN loss is undefined at the generalized solution; therefore, it is not a minimizer of \eqref{eq:pinn}.
Even the existence of a minimizer is uncertain if the PDE only has a generalized solution rather than a classical solution.
However, some prior studies have still assumed this premise.

Under this premise, PINNs represent the solution over a neural network and sequentially adjust the network by minimizing the PINN loss to enforce it to approximate the solution.
Accordingly, PINNs rely on two observable measures to determine whether training is complete.
The loss is evaluated to check whether it is sufficiently close to zero. 
When the loss reaches a satisfactory level and stabilizes, it can be judged that the model is sufficiently trained.  
Additionally, the convergence of the network is verified because, even if the loss function has been minimized, the network may still be oscillating.
To this end, the output values of the network are monitored to see if they have reached a stable state where they are no longer undergoing significant changes during each iteration.
Mathematically, we use the following criteria to solve PINNs:
\begin{criterion}\label{crit:loss}
    Does the loss converge to zero?
\end{criterion}
\begin{criterion}\label{crit:network}
    Does the decision variable converge in $L^p\left(\Omega,\bR\right)$\footnote{We aim to converge to a generalized solution that belongs to $W^{1,p}$. Therefore, it may seem appropriate to consider the $W^{1,p}$-norm rather than $L^p$ in Criterion \ref{crit:network}.
    However, in practical implementations, the convergence of the derivatives is agonistic to oversee.
    Additionally, we propose a novel method in Section \ref{sec:variable_splitting} where convergence is ensured by Criterion \ref{crit:network} with $L^p$ convergence only. Furthermore, as the convergence of the network is typically evaluated by observing the output values, one might think that it is appropriate to judge convergence using $L^\infty$-norm. However, it is more sensible to use the $L^p$-norm, which is a metric that measures the difference from the target solution, because there exists a generalized solution that is not essentially bounded.}?
\end{criterion}

However, we demonstrate in the following section that the premise is invalid, as satisfying Criteria \ref{crit:loss} and \ref{crit:network} does not guarantee the convergence of the network to a desired PDE solution.
Furthermore, in Section \ref{sec:variable_splitting}, we propose a novel methodology that ensures convergence under these two criteria.

\section{Failure Mode of PINNs}\label{sec:pinn_failure}
Despite their impressive practical success, PINNs have failed to produce accurate predictions in several cases.
In this section, we derive a new insight into the fundamental mode of failure of PINNs based on two primary findings:

\begin{enumerate}
    \item Theorem \ref{thm:pinn_failure} shows an incompleteness of PINNs that a convergent sequence of the decision variable reducing the loss function to zero does not necessarily converge to a PDE solution.
    \item Remark \ref{rem:gradient_pinn} unravels that the failure of PINNs in Theorem \ref{thm:pinn_failure} is attributed from that the derivative of the predicted solution explodes. 
    In other words, PINNs lack the convergence guarantees for the derivatives.
\end{enumerate}

The first findings addresses a fundamental question on PINNs: Does the Premise of residual minimization hold?
In Theorem \ref{thm:pinn_failure}, we demonstrate that PINNs cannot guarantee convergence to the solution of the governing PDE, even if Criteria \ref{crit:loss} and \ref{crit:network} are satisfied. 

\begin{theorem}\label{thm:pinn_failure}
    For any $1\le p<\infty$ and $d\in\bN$, there exists a domain $\Omega\subset\bR^d$, coefficients $a_{ij}, b_i, c\in C^1\left(\Omega,\bR\right)$ for $1\le i,j\le d$, and source functions $f\in C^1\left(\Omega,\bR\right)$ and $g\in C\left(\partial\Omega,\bR\right)$ with the following properties:
    \begin{enumerate}
        \item Assumptions \ref{assump:domain}---\ref{assump:sol} hold and the second-order linear PDE \eqref{eq:pde} has a classical solution $u^\ast\in C^2\left(\Omega,\bR\right)\cap C\left(\bar{\Omega},\bR\right)$.
        Moreover, $u^\ast$ is a unique generalized solution.
        \item There exists a convergent sequence $u_n\in C^2\left(\bar{\Omega},\bR\right)$ in $L^p\left(\Omega\right)$ such that
        \begin{equation}
            \cL_p\left(u_n\right)\rightarrow 0 \text{ as } n\rightarrow\infty,
        \end{equation}
        but the limit of $u_n$ is not a solution of \eqref{eq:pde}.
    \end{enumerate}
\end{theorem}
\begin{proof}
    Let $\Omega=B_1\left(0\right)$, and define the operator $\cN\left[u,Du,D^2u\right] (x)= \left(1-\left|x\right|\right)^2D\cdot D u+\tau\left( x\right)^\text{T} Du$, where $\tau\in C^1\left(\Omega,\bR\right)$ with $\tau(x)\perp x$.
    Note that as the principal coefficient is differentiable, $\cN$ can be rewritten as the divergence form as follows:
        \begin{equation}
            \cN\left[u,Du,D^2u\right]=D\cdot\left(\left(1-\left|x\right|\right)^2 Du\right) + \left(\tau\left(x\right) + 2\frac{1-\left|x\right|}{\left\vert x\right\vert}x\right)^\text{T} Du.
        \end{equation}
        For the sake of simplicity in subsequent calculations, we shall use the non-divergence form.
    Then the following PDE
    \begin{equation}\label{eq:pf_thm_pinn_fail_pde}
        \begin{cases}
            \cN\left[u,Du,D^2u\right]=0,
            &\text{ in }\Omega,
            \\
            \hfill u=0,
            &\text{ on }\partial\Omega
        \end{cases}
    \end{equation}
    has a trivial solution $u^\ast=0$.

    Motivated from the proof of Lemma 2.9 in \cite{kim2023generalized}, we show that $u^*$ is a unique generalized solution.
    Suppose $u$ is a generalized solution of the PDE \eqref{eq:pf_thm_pinn_fail_pde}.
    Since the principal coefficient $a_{ij}$ is H\"older continuous, the interior regularity and the Schauder estimates implies that $u$ is in $C^2\left(\Omega\right)\cap C\left(\bar{\Omega}\right)$ \citep{evans2022partial, han2011elliptic}.
    For any $0<r<1$, the operator $\cN$ is uniformly elliptic in $B_r\left(0\right)$.
    Thus, we can apply the weak maximum principle in \cite{han2011elliptic} and hence $u$ cannot attain its maximum or minimum in $B_r\left(0\right)$.
    In other words, a generalized solution $u$ of the PDE \eqref{eq:pf_thm_pinn_fail_pde} assumes its maximum or minimum on the boundary $\partial\Omega$ unless $u$ is constantly zero.

    However, we will construct a sequence $u_n\in C^2\left(\Omega,\bR\right)\cap C\left(\bar{\Omega}, \bR\right)$ such that 
    \begin{equation}
     \cL_p\left(u_n\right)\rightarrow 0 \text{ as } n\rightarrow\infty,
    \end{equation}
    and $u_n$ converges to a function $u$ in $L^p\left(\Omega\right)$, but
    $u$ is not a solution to \eqref{eq:pf_thm_pinn_fail_pde}.
    \\
    For each $n\in\bN$, set $r_n=1-\frac{1}{n}$ and define $u_n\left(x\right)=\rho_n\left(r\right)$, where $r=\left\vert x\right\vert$ is a radius in spherical coordinates of $x$;
       \begin{equation}
        \rho_n\left(r\right)=
        \begin{cases}
            \hfil 1,
            & \text{ if } r \leq r_n,
            \\
            \left(1-r\right) q_n\left(r\right),
            & \text{ if } r > r_n,
        \end{cases}   \end{equation}
    with a quadratic polynomial $q_n$ defined as
    \begin{equation}
        q_n\left(r\right)=n^3\left( r-r_n\right)^2+ n^2\left( r-r_n\right) + n.
    \end{equation}
    Direct calculation induces that $u_n\in C^2\left(\Omega,\bR\right)\cap C\left(\bar{\Omega},\bR\right)$ and we find that for $r\in\left[r_n, 1\right]$,
    \begin{align}
        \left\vert q_n(r) \right\vert & \leq n^3\left( 1-r_n\right)^2+ n^2\left( 1-r_n\right) + n \leq 3n \label{eq:pf_pinn_fail_bound_qk}\\
        \left\vert q_n'(r) \right\vert & = \left\vert 2n^3\left(r-r_n\right) +n^2 \right\vert  \leq 2n^3\left(1-r_n\right) +n^2 \leq 3n^2, \label{eq:pf_pinn_fail_bound_Dqk}\\
        \left\vert q_n''(r)\right\vert&=2n^3 \label{eq:pf_pinn_fail_bound_D2qk}.
    \end{align}

    We first confirm that PINN loss converges to zero.    
    Change of variable to spherical coordinates reformulates the Laplacian as 
    \begin{align}
                D\cdot D u_n\left(x\right)
                &= \rho_n''\left(\left\vert x\right\vert\right)+\frac{d-1}{\left\vert x\right\vert}\rho_n'\left(\left\vert x\right\vert\right)
                \\
                & =  -2q_n'\left(\left\vert x\right\vert\right) + \left(1-\left\vert x\right\vert\right)q_n''\left(r\right) + \frac{d-1}{\left\vert x\right\vert}\left(-q_n\left(\left\vert x\right\vert\right) + \left(1-\left\vert x\right\vert\right)q_n'\left(\left\vert x\right\vert\right)\right).
    \end{align}
    For any $r\in\left[r_n,1\right]$, it follows that 
            \begin{align}
                &\left\vert-2q_n'\left(r\right) + \left(1-r\right)q_n''\left(r\right) + \frac{d-1}{r}\left(-q_n\left(r\right) + \left(1-r\right)q_n'\left(r\right)\right) \right\vert\\
                & \leq 2\left\vert q_n'\left(r\right)\right\vert + \left(1-r_n\right)\left\vert q_n''\left(r\right)\right\vert + \frac{d-1}{r_n}\left\vert q_n\left(r\right)\right\vert + \left(d-1\right)\frac{1-r_n}{r_n}\left\vert q_n'\left(r\right)\right\vert \\
                & \leq 2\left\vert q_n'\left(r\right)\right\vert + n^{-1}\left\vert q_n''\left(r\right)\right\vert + 2(d-1)\left\vert q_n\left(r\right)\right\vert + \left(d-1\right)\left\vert q_n'\left(r\right)\right\vert \\
                & \leq 6n^2 + 2n^2 + 6(d-1)k + 3(d-1)n^2\\
                & \leq 6n^2 + 2n^2 + 6(d-1)n^2 + 3(d-1)n^2 \\
                & = (9d-1)n^2,
            \end{align}
            where the third inequality follows from \eqref{eq:pf_pinn_fail_bound_qk}, \eqref{eq:pf_pinn_fail_bound_Dqk}, and \eqref{eq:pf_pinn_fail_bound_D2qk}.
            
            In addition, in spherical coordinate, the first-order term in the PDE \eqref{eq:pf_thm_pinn_fail_pde} is rewritten as 
            \begin{equation}
                \tau\left(x\right)^\text{T}Du=\frac{x}{\left\vert x\right\vert}^\text{T}\tau\left(x\right) \rho'(\left\vert x\right\vert).
            \end{equation}
            From $\tau\left(x\right)\perp x$, we obtain 
            \begin{equation}
                \tau\left(x\right)^\text{T}Du=0.
            \end{equation}
    Consequently, we have
    \begin{align}
        \frac{1}{d\left\vert\Omega\right\vert}\int_\Omega \left\vert \cN\left[u_n,Du_n,D^2u_n\right]\right\vert^p & =  \int_{r_n}^1 \left\vert (1-r)^2\left( \rho_n''(r)+\frac{d-1}{r}\rho_n'(r)\right)\right\vert^pr^{d-1} \diff r\\
        & \leq   \int_{r_n}^1 \left\vert (1-r)^2\left( \rho_n''(r)+\frac{d-1}{r}\rho_n'(r)\right)\right\vert^p \diff r\\
        & \leq \left( (9d-1)n^2 \right)^p\int_{r_n}^1 (1-r)^{2p} \diff r\\
        & = \frac{(9d-1)^p}{2p+1}n^{2p}\left(1-r_n\right)^{2p+1}\\
        &= \frac{(9d-1)^p}{2p+1}n^{-1}\\
        & \rightarrow 0,  
    \end{align}
    as $n\rightarrow \infty$.
    In addition, it is clear that that $\left\Vert u_n \right\Vert_{L^p\left(\partial\Omega\right)}\rightarrow 0$ as $n\rightarrow\infty$. 
    As a consequence, we have 
    \begin{equation}
        \cL_p\left(u_n\right)\rightarrow 0 \text{ as } n\rightarrow\infty.
    \end{equation}

    We now verify the convergence of the sequence $u_n$.
    It is straightforward to show that the sequence $u_n$ converges point-wisely to a function $1 \neq u^\ast$, which implies the convergence in $L^p\left(\Omega\right)$ by the dominated convergence theorem.
\end{proof}

Theorem \ref{thm:pinn_failure} poses a fundamental problem for PINNs before delving into the training concerns. 
The utilization of definite integrals in the PINN loss \eqref{eq:pinn_loss} in the proof prevents failures caused by approximation errors from the numerical integration in the loss formulation. 
Because the loss converges to zero, the decision variable is not trapped in a spurious local, indicating no optimization issues.
Accordingly, the theorem shows that Premise on residual minimization does not hold.

Furthermore, Theorem \ref{thm:pinn_failure} leads to a new understanding of the connection between the failure modes of PINNs and the non-convergence of the gradient.
To determine what causes PINNs to fail to converge, we scrutinized the sequence constructed in the proof of Theorem \ref{thm:pinn_failure}.
Specifically, we examine the asymptotic behavior of the gradient and demonstrate that it is uncontrolled within the PINN framework.
In the following remark, we shed new light on the cause of the failure of PINNs, called \textit{derivative pathology}.

\begin{remark}[Derivative pathology of PINNs]\label{rem:gradient_pinn}
To characterize the cause of the failure, we focus on the entanglement of $D^\alpha u$ in the PINN loss for $\alpha=0,1,2$.
Reducing the loss of PINNs may not necessarily result in control of each derivative of the decision variable.
In other words, the loss can be sufficiently small even if one or more $D^\alpha u$ do not converge properly.
We call this phenomenon ``derivative pahology'' of PINNs.
Theorem \ref{thm:pinn_failure} and its proof show that the derivative pathology occurs, as the derivative of $u_n$ explodes:
    \begin{align}
        & \frac{1}{d\left\vert\Omega\right\vert }\int_\Omega \left\vert Du_n \right\vert^p \\
        &\quad = \int_{r_n}^1\left\vert -q_n\left(r\right) + \left(1-r\right)q_n'\left(r\right)\right\vert^p r^{d-1} \diff r\\
        &\quad = \int_{r_n}^1 \left\vert -\left(n^3\left( r-r_n\right)^2+n^2\left( r-r_n\right)+n\right) + (1-r)\left(2n^3\left(r-r_n\right) +n^2\right)\right\vert^p r^{d-1} \diff r\\
        &\quad = \int_{r_n}^1 \left\vert -3n^3r^2 + 6n^2(n-1)r + n^3-3n\right\vert^p r^{d-1} \diff r\\
        &\quad = \int_{r_n}^1 \left\vert -3n^3\left(r-\frac{n-1}{n}\right)^2 + 4n^3-6n^2\right\vert^p r^{d-1} \diff r\\
        &\quad  = \int_{r_n}^1 \left\vert -3n^3\left(r-r_n\right)^2 + 4n^3-6n^2\right\vert^p r^{d-1} \diff r\\
        &\quad \geq \left\vert -3n^3\left(1-r_n\right)^2 + 4n^3-6n^2\right\vert^p \int_{r_n}^1  r^{d-1} \diff r\\
        &\quad =\left\vert -3n + 4n^3-6n^2\right\vert^p \frac{1-\left(1-n^{-1}\right)^d}{d}\\
        &\quad \rightarrow \infty,\ \text{as }n\rightarrow\infty.
    \end{align}
    Considering the generalized solution in $W^{1,p}\left(\Omega,\bR\right)$, the gradient of the predicted solution should be integrable. 
    However, PINNs often fail to guarantee convergence of the gradient and may even allow it to blow up.
    This derivative pathology provides a novel perspective on the failure modes of PINNs, emphasizing the significance of the behavior of the derivatives of the predicted solution, which has been underestimated.
\end{remark}

Remark \ref{rem:gradient_pinn} leads to a new understanding of the connection between the failure modes of PINNs and the non-convergence of the derivatives.
While there have been efforts to identify the failure modes of PINNs, a comprehensive understanding of the aspect of the gradient of the decision variable remains largely unexplored.
However, in Remark \ref{rem:gradient_pinn}, we find that PINNs have an incomplete structure that cannot direct the behavior of the derivatives of the approximator, which impedes the convergence of the gradient and ultimately hinders the convergence to the PDE solution.
Thus, the derivative pathology raised in Remark \ref{rem:gradient_pinn} sheds new light on the cause of the failure of PINNs.
It insists on the importance of considering both the function value and its derivatives when solving PINNs.

In the next section, we present a method for monitoring the gradient through an auxiliary variable and establish the convergence of the proposed method to a generalized solution. 
Furthermore, we discuss that the pathological behavior of the gradient is the culprit for the failure of PINNs to converge.

\section{Variable Splitting}\label{sec:variable_splitting}
This section introduces a novel method that rectifies the deficiency of PINNs and demonstrates that it converges to a generalized solution of second-order linear PDEs.
Our method is inspired by the insights into the causes of failure in PINNs, as discussed in the previous section.
As illustrated in Theorem \ref{thm:pinn_failure}, fulfilling Criteria \ref{crit:loss} and \ref{crit:network} does not necessarily guarantee that the limit of the decision variable of PINNs is a solution of the governing PDE. We further uncover that the limitation stems from the inability to ensure the convergence of the gradient in Remark \ref{rem:gradient_pinn}.
To overcome the flaws of PINNs and monitor the gradient directly, we propose a \textit{variable splitting strategy} that employs an auxiliary variable to approximate the gradient of the solution.
We demonstrate that the proposed method converges to a generalized solution of second-order linear PDEs.
Moreover, We discuss additional strengths and practical implications of variable splitting.

\subsection{Method}\label{subsec:vs_method}
We consider second-order linear PDEs \eqref{eq:pde}.
Distinct from PINNs that configure a decision variable with the PDE solution, we consider not only the primary variable that approximates the solution but also the auxiliary variable that approximates the gradient of the solution.
Accordingly, variable splitting uses the following two variables, $u$ and $V$, with constraint $Du=V$:
\begin{itemize}
    \item \textbf{Primary variable $u:\bar{\Omega}\rightarrow\bR$} to approximate a solution of PDEs.
    \item \textbf{Auxiliary variable $V: \Omega\rightarrow\bR^d$} to approximate the gradient of a solution of PDEs.
\end{itemize}
Defining the auxiliary variable $V$ to represent the gradient of the solution reformulates the PDE \eqref{eq:pde} into a system of first-order PDEs:
\begin{equation}\label{eq:pde_split}
    \begin{cases}
        D u = V, & \text{ in }\Omega,\\
        \cN\left[u,V,DV\right]=-D\cdot\left(AV\right) + b^\text{T} V + cu = f, & \text{ in }\Omega,\\
        u = g, & \text{ on }\partial\Omega.
    \end{cases}
\end{equation}
By penalizing each equation in \eqref{eq:pde_split} as analogous to PINNs, we enforce the primary $u$ and auxiliary variable $V$ to learn the solution and its gradient by minimizing the following loss function of variable splitting:
\begin{definition}[Variable Splitting for PINNs]\label{def:variable_splitting}
    Suppose a second-order linear PDE \eqref{eq:pde} is defined in a domain $\Omega\subset\bR^d$ and $1\leq p\leq\infty$.
    \textbf{Variable splitting for PINNs (VS-PINNs)} approximate a solution of \eqref{eq:pde} by minimizing the residual of \eqref{eq:pde_split}.
    \textbf{$L^p$ loss function of the VS-PINNs} 
            \begin{equation}
                \cL^{VS}_p:W^{1,p}\left(\Omega,\bR\right)\times W^{1,p}\left(\Omega,\bR^d\right)\rightarrow \bR_+
            \end{equation} 
            is defined by the linear combination of three loss terms 
            \begin{equation}\label{eq:splitting_loss}
                \cL^{VS}_p \left(u,V\right) = \lambda_{\cN}\cL_{\cN,p}\left(u,V\right) + \lambda_D\cL_{D,p} \left(u,V\right) + \lambda_\cB\cL_{\cB,p}\left(u\right),
            \end{equation}
            where $u\in W^{1,p}\left(\Omega,\bR\right)$, $V\in W^{1,p}\left(\Omega,\bR^d\right)$, $\lambda_{\cN}, \lambda_D,\lambda_\cB>0$ are regularization parameters,
            and each loss term is defined as follows:
            \begin{enumerate}[1.]
                \item The first term $\cL_{\cN,p}$ is a PDE loss 
                    \begin{equation}\label{eq:vs_loss_pde}
                        \cL_{\cN,p}\left(u,V\right) = \left\Vert\cN \left[u,V,DV\right]-f\right\Vert_{L^p\left(\Omega\right)},
                    \end{equation}
                    with the differential operator $\cN$ given by \eqref{eq:div_form}.
                \item The second term $\cL_{D,p}$ is a gradient matching loss defined as 
                    \begin{equation}\label{eq:vs_loss_gm}
                         \cL_{D,p}\left(u,V\right) =\left\Vert D u-V \right\Vert_{L^p\left(\Omega\right)}.
                    \end{equation}
                \item The last term $\cL_{\cB,p}$ corresponds to a boundary loss which is identical to \eqref{eq:pinn_loss_bdy}.
            \end{enumerate}
\end{definition}

Note that the variable splitting strategy splits the PDE loss \eqref{eq:pinn_loss_pde} of PINNs into two losses, \eqref{eq:vs_loss_pde} and \eqref{eq:vs_loss_gm}.
The auxiliary variable $V$ replaces $Du$ in the PDE loss term \eqref{eq:pinn_loss_pde} and yields the first loss term \eqref{eq:vs_loss_pde}, which forces the variables $u$ and $V$ to abide by the laws of physics.
The second term \eqref{eq:vs_loss_gm}, which we refer to as the gradient matching loss following \cite{park2023p}, ensures that the primary and auxiliary variables adhere to the constraint $Du=V$. 

\begin{remark}\label{rem:split_higher_order}
In this study, we focus on second-order linear PDEs \eqref{eq:pde} to analyze their convergence.
However, it is worth noting that variable splitting does not need to be confined to second-order linear PDEs \eqref{eq:pde}. VS-PINNs can be easily generalized to PDEs of the general form with any boundary condition without paying attention to other types of PDEs, similar to PINNs.
For higher-order PDEs, variable splitting can be employed in various ways. 
For $(k+1)$-th order PDEs, one may augment (i) only the derivative before the highest order $V=D^{k}u$ of the solution as an auxiliary variable, or (ii) all derivatives from the first to the right before the highest derivative as auxiliary variables, that is, $V_1=Du,V_2=D^2u,\cdots,V_{k}=D^ku$.
When considering memory consumption and computational cost, (ii) would be more efficient; it splits the higher-order PDE into a system of first-order PDEs by replacing all the derivatives as auxiliary variables. 
Additionally, incorporating all derivatives as decision variables is anticipated to be more effective, allowing for direct regulation of all derivatives.
\end{remark}

\subsection{Convergence Analysis for Variable Splitting}\label{sec:convergence_VS}
We establish a convergence analysis of the proposed variable splitting strategy for a generalized solution. More specifically, in contrast to PINNs, which cannot obtain a solution based on Criteria \ref{crit:loss} and \ref{crit:network}, the primary variable of the VS-PINNs converges to the generalized solution when the auxiliary variable converges in $L^p\left(\Omega,\bR^d\right)$ and the $L^p$ VS-PINN loss converges to zero.

Recall that the failure of PINNs to converge is related to their inability to achieve convergence of the decision variable in $W^{1,p}\left(\Omega,\bR\right)$ from Criteria \ref{crit:loss} and \ref{crit:network}.
From this perspective, we prove the convergence of VS-PINNs in two steps: Theorem \ref{thm:u_convergence_from_V} guarantees the convergence of the primary variable $u$ in $W^{1,p}\left(\Omega,\bR\right)$, and Theorem \ref{thm:conv_weak_2ndlinear} confirms that the limit is a generalized solution.
Specifically, our main results are twofold:
\begin{enumerate}
    \item Theorem \ref{thm:u_convergence_from_V} shows that the convergence of the auxiliary variable $V$ ensures not only the convergence of the gradient of $u$ but also the convergence of $u$.
    This implies that introducing the auxiliary variable bypasses the primary cause of the failure mode of PINNs.
    \item Theorem \ref{thm:conv_weak_2ndlinear} guarantees the convergence of the primary variable $u$ of variable splitting to a solution of the PDE \eqref{eq:pde} provided that variables reduce the loss $\cL^{VS}_p\left(u,V\right)$ to zero and $V$ converges.
\end{enumerate}

We begin by deriving a Poincar\'e-type inequality that describes the relationship between the primary variable $u$, its gradient $Du$, and boundary values needed to prove Theorem \ref{thm:u_convergence_from_V}.

\begin{lemma}\label{lem:Poincare_bdy}
    Let $1\le p\le\infty$ and $u\in W^{1,p}\left(\Omega,\bR\right)\cap C\left(\bar{\Omega},\bR\right)$.
    Then under Assumption \ref{assump:domain}, there exists a constant $C=C\left(\Omega,p\right)$ depending only on $\Omega$ and $p$, such that
    \begin{equation}\label{eq:Poincare_bdy}
        \left\Vert u-\langle u\rangle_{\partial\Omega}\right\Vert_{L^p\left(\Omega\right)}\leq C\left\Vert D u\right\Vert_{L^p\left(\Omega\right)},
    \end{equation}
    where 
    \begin{equation}\label{eq:average_on_boundary}
        \langle u\rangle_{\partial\Omega}=\frac{1}{\left\vert \partial\Omega \right\vert}\int_{\partial\Omega}u \diff x.
    \end{equation}
\end{lemma}

\begin{proof}
    We prove \eqref{eq:Poincare_bdy} by contradiction, following the proof of Poincar\'e inequality in \cite{evans2022partial}. Suppose not, then there exists a sequence $\left\{u_n\right\}\subset W^{1,p}\left(\Omega\right)$ such that 
    \begin{equation}\label{eq:poincare_pf_seq}
        \left\Vert u_n-\langle u_n\rangle_{\partial\Omega} \right\Vert_{L^p\left(\Omega\right)} > n\left\Vert D u_n\right\Vert_{L^p\left(\Omega\right)}.
    \end{equation}
    We normalize $u_n$ by defining 
    \begin{equation}
        v_n= \frac{u_n-\langle u_n\rangle_{\partial\Omega}}{\left\Vert u_n-\langle u_n\rangle_{\partial\Omega} \right\Vert_{L^p\left(\Omega\right)}}.
    \end{equation}
    Then, it follows
    \begin{equation}\label{eq:poincare_pf_v_cond}
       \left\Vert v_n \right\Vert_{L^p\left(\Omega\right)}=1,\ \langle v_n\rangle_{\partial\Omega}=0.
    \end{equation}
    Moreover, \eqref{eq:poincare_pf_seq} induces that 
    \begin{equation}\label{eq:poincare_pf_Dv_cond}
        \left\Vert D v_n \right\Vert_{L^p\left(\Omega\right)}<\frac{1}{n}\left\Vert v_n-\langle v_n\rangle_{\partial\Omega} \right\Vert_{L^p\left(\Omega\right)}=\frac{1}{n}.
    \end{equation}
    Since $\left\Vert v_n\right\Vert_{W^{1,p}\left(\Omega\right)}<1+\frac{1}{n}$ is uniformly bounded and $W^{1,p}\left(\Omega\right)$ is compactly embedded in $L^p\left(\Omega\right)$, there exists a subsequence $\left\{v_{n_j}\right\}_{j=1}^\infty\subset \left\{v_n\right\}_{n=1}^\infty$  satisfying
    \begin{equation}\label{eq:poincare_pf_v_converge}
        v_{n_j}\rightarrow v \text{ in } L^p\left(\Omega\right),
    \end{equation}
    as $j\rightarrow \infty$ for some $v\in L^p\left(\Omega\right)$.
    Thus, we may assume $v_n$ converges to $v$ in $L^p\left(\Omega\right)$, taking a subsequence if needed.
    
    In addition, \eqref{eq:poincare_pf_Dv_cond} and \eqref{eq:poincare_pf_v_converge} imply for each $n\in\bN$ that for every $\phi\in C_c^\infty\left(\Omega,\bR\right)$
    with the H\"older conjugate $q=\frac{p}{p-1}\in\left[1,\infty\right]$ of $p$, the H\"older inequality induces 
    \begin{align}
        \left\vert\int_\Omega vD\phi \diff x\right\vert & \leq  \left\vert\int_\Omega \left(v-v_n\right)D\phi \diff x\right\vert+\left\vert\int_\Omega v_nD\phi \diff x\right\vert \\
        & \leq \int_\Omega \left\vert\left(v-v_n\right)D\phi\right\vert \diff x + \int_\Omega\left\vert \phi Dv_n\right\vert \diff x \\
        & \leq \left\Vert v-v_n\right\Vert_{L^p\left(\Omega\right)}\left\Vert D\phi\right\Vert_{L^q\left(\Omega\right)} + \left\Vert Dv_n\right\Vert_{L^p\left(\Omega\right)}\left\Vert\phi\right\Vert_{L^q\left(\Omega\right)}\\
        & \rightarrow 0,
    \end{align}
    as $n\rightarrow\infty$.
    As a consequence, $v\in W^{1,p}\left(\Omega\right)$ with $Dv=0$ and $v_n$ converges to $v$ in $W^{1,p}\left(\Omega\right)$. 
    As $\Omega$ is connected and $v\in W^{1,p}\left(\Omega\right)$, 
    $v\equiv\bar{v}$ in $\Omega$ for some constant $\bar{v}\in\bR$.

    If $p=\infty$, uniform convergence of $Dv_n$ on $\Omega$ implies uniform convergence of $v_n$ on $\bar{\Omega}$.
    Thus, constant function $\bar{v}\in C\left(\bar{\Omega},\bR\right)$ satisfies both in \eqref{eq:poincare_pf_v_cond}, which contradicts to each other.
    
    Now let $1\le p<\infty$. Since the trace operator $T:W^{1,p}\left(\Omega,\bR\right)\rightarrow L^p\left(\partial\Omega,\bR\right)$ is a continuous operator such that
    \begin{equation}
        Tu=u\mid_{\partial\Omega},\ \forall u\in W^{1,p}\left(\Omega\right)\cap C\left(\bar{\Omega}\right),
    \end{equation}
    we have $T\bar{v}=\bar{v}$ and it follows that
    \begin{align}
        \left\vert \bar{v}-\langle  v_n\rangle_{\partial\Omega}\right\vert^p &= \left\vert\frac{1}{\left\vert\partial\Omega\right\vert}\int_{\partial\Omega} v_n-\bar{v} \diff x\right\vert^p\\
        &\leq \frac{1}{\left\vert\partial\Omega\right\vert} \int_{\partial\Omega}\left\vert v_n-\bar{v}\right\vert^p \diff x\\
        &\leq \frac{1}{\left\vert\partial\Omega\right\vert}\left\Vert v_n-\bar{v}\right\Vert_{L^p\left(\partial\Omega\right)}^p\\
        &= \frac{1}{\left\vert\partial\Omega\right\vert} \left\Vert Tv_n-T\bar{v}\right\Vert_{L^p\left(\partial\Omega\right)}^p\\
        &\leq \frac{1}{\left\vert\partial\Omega\right\vert}\left\Vert T\right\Vert\left\Vert v_n-\bar{v}\right\Vert_{W^{1,p}\left(\Omega\right)}^p\\
        &
        = \frac{1}{\left\vert\partial\Omega\right\vert}\left\Vert T\right\Vert\left(\left\Vert v_n-\bar{v}\right\Vert_{L^p\left(\Omega\right)}^p+\left\Vert Dv_n\right\Vert_{L^p\left(\Omega\right)}^p\right)\\
        &\rightarrow0,
    \end{align}
    as $n\rightarrow\infty$. We use Jensen's inequality on $\partial\Omega$ and the boundedness of $T$.
    Since $\langle v_n\rangle_{\partial\Omega}=0$ for all $n\in\bN$ by definition, $\bar{v}=0$ and hence $\left\Vert v\right\Vert_{L^p\left(\Omega\right)}=0$, which contradicts to $\left\Vert v \right\Vert_{L^p\left(\Omega\right)}=1$.
\end{proof}

The significance of this lemma is that one can bound the $L^p$-norm of $u$ minus a constant obtained from the boundary values using the $L^p$-norm of the gradient $Du$.
This indicates that the convergence of $u$ is induced by the convergence of $Du$ and the boundary values.

In the following theorem, we leverage the convergence of $V$, the gradient matching loss, and the boundary loss to prove the convergence of $u$ in the Sobolev norm using the above lemma.

\begin{theorem}\label{thm:u_convergence_from_V}
    Suppose $\Omega$ satisfies Assumption \ref{assump:domain}.
    For given $1\le p\le \infty$ and $g\in L^p\left(\partial\Omega,\bR\right)$,
    let $u_n\in W^{1,p}\left(\Omega,\bR\right)\cap C\left(\bar{\Omega},\bR\right)$ and $V_n\in L^p\left(\Omega,\bR^d\right)$ be sequences such that
    \begin{align}
        \left\Vert Du_n-V_n\right\Vert_{L^p\left(\Omega\right)} &\rightarrow0,
        \\
        \left\Vert u_n-g\right\Vert_{L^p\left(\partial\Omega\right)} &\rightarrow0,
    \end{align}
    as $n\rightarrow\infty$.
    If $V_n$ converges in $L^p\left(\Omega\right)$ then $u_n$ converges in $W^{1,p}\left(\Omega\right)$.
\end{theorem}

\begin{proof}
    Let $V\in L^p\left(\Omega\right)$ be the limit point of $V_n$.
    For an arbitrary $\varepsilon>0$, there exists $N\in\bN$ such that
    \begin{align}
        \left\Vert V_n-V\right\Vert_{L^p\left(\Omega\right)} &< \varepsilon,
        \\
        \left\Vert Du_n-V_n\right\Vert_{L^p\left(\Omega\right)} &< \varepsilon,
        \\
        \left\Vert u_n-g\right\Vert_{L^p\left(\partial\Omega\right)} &< \varepsilon,
    \end{align}
    for all $n\geq N$.
    Then $Du_n$ converges to $V$, since $n\ge N$ implies
    \begin{equation}
        \left\Vert Du_n-V\right\Vert_{L^p\left(\Omega\right)}\leq \left\Vert Du_n-V_n\right\Vert_{L^p\left(\Omega\right)} + \left\Vert V_n-V\right\Vert_{L^p\left(\Omega\right)} < 2\varepsilon.
    \end{equation}
    
    To prove that $u_n$ converges in $L^p$, we use Lemma \ref{lem:Poincare_bdy} and show that $u_n$ is a Cauchy sequence in $L^p\left(\Omega\right)$.
    First, we show that $\left\vert\langle u_n\rangle_{\partial\Omega}-\langle g\rangle_{\partial\Omega}\right\vert$ converges to $0$, where $\langle\cdot\rangle_{\partial\Omega}$ is defined as \eqref{eq:average_on_boundary}.
    It is obvious for $p=\infty$ because the boundary loss induces uniform convergence on $\partial\Omega$.
    For the case of $1\le p<\infty$, the Jensen's inequality induces the same result:
    \begin{align}
        \left\vert\langle u_n\rangle_{\partial\Omega}-\langle g\rangle_{\partial\Omega}\right\vert
        &= \left( \left\vert \frac{1}{\left\vert\partial\Omega\right\vert} \int_{\partial\Omega} u_n-g \diff x\right\vert^p \right) ^{1/p}
        \\
        &\le \left( \frac{1}{\left\vert\partial\Omega\right\vert}\int_{\partial\Omega}\left\vert u_n-g\right\vert^p \diff x \right)^{1/p}
        \\
        &=\frac{1}{\left\vert\partial\Omega\right\vert^{1/p}}\left\Vert u_n-g\right\Vert_{L^p\left(\partial\Omega\right)},
        \\
        &< \frac{1}{\left\vert\partial\Omega\right\vert^{1/p}}\varepsilon.
    \end{align}
    Therefore, for any $p\in\left[1,\infty\right]$ and $n,m\geq N$, we attain
    \begin{equation}
        \left\Vert \langle u_n\rangle_{\partial\Omega} - \langle u_m\rangle_{\partial\Omega}\right\Vert_{L^p\left(\Omega\right)} \leq \left\Vert\langle u_n\rangle_{\partial\Omega}-\langle g\rangle_{\partial\Omega}\right\Vert_{L^p\left(\Omega\right)} + \left\Vert\langle u_m\rangle_{\partial\Omega}-\langle g\rangle_{\partial\Omega}\right\Vert_{L^p\left(\Omega\right)} < 2\varepsilon,
    \end{equation}
    and
    \begin{align}
        \left\Vert u_n - u_m\right\Vert_{L^p\left(\Omega\right)} & \leq \left\Vert u_n -u_m - \left(\langle u_n\rangle_{\partial\Omega} - \langle u_m\rangle_{\partial\Omega}\right)\right\Vert_{L^p\left(\Omega\right)} + \left\Vert \langle u_n\rangle_{\partial\Omega} - \langle u_m\rangle_{\partial\Omega}\right\Vert_{L^p\left(\Omega\right)}\\
        & < \left\Vert u_n -u_m - \left(\langle u_n\rangle_{\partial\Omega} - \langle u_m\rangle_{\partial\Omega}\right)\right\Vert_{L^p\left(\Omega\right)} + 2\varepsilon\\
        & \leq C_1\left\Vert Du_n-Du_m\right\Vert_{L^p\left(\Omega\right)} + 2\varepsilon\\
        & \leq C_1\left\Vert Du_n-V\right\Vert_{L^p\left(\Omega\right)} + C_1\left\Vert Du_m-V\right\Vert_{L^p\left(\Omega\right)} + 2\varepsilon\\
        & < 2\left(C_1+1\right)\varepsilon,
    \end{align}
    for some constant $C_1=C_1\left(\Omega,p\right)$, which appears in Lemma \ref{lem:Poincare_bdy}.
    Therefore, $u_n$ converges to $u\in L^p\left(\Omega\right)$. 
    
    Indeed, $u\in W^{1,p}\left(\Omega\right)$ and $Du_n$ converges to $Du=V$ in $L^p\left(\Omega\right)$, because for any fixed test function $\phi\in C_c^\infty\left(\Omega\right)$ and $n\geq N$, we have
    \begin{align}
        \left| \int_\Omega V\phi+uD\phi \diff x\right|
        &= \left|\int_\Omega\left(V-Du_n\right)\phi+ \left(u-u_n\right)D\phi\diff x\right|\\
        &\leq  \left|\int_\Omega\left(V-Du_n\right)\phi\diff x\right| + \left|\int_\Omega\left(u-u_n\right)D\phi\diff x\right|\\
        &\leq \int_\Omega\left|\left(V-Du_n\right)\phi\right|\diff x + \int_\Omega\left|\left(u-u_n\right)D\phi\right|\diff x\\
        &\leq \left\Vert V-Du_n\right\Vert_{L^p\left(\Omega\right)}\left\Vert\phi\right\Vert_{L^q\left(\Omega\right)} + \left\Vert u-u_n\right\Vert_{L^p\left(\Omega\right)}\left\Vert D\phi\right\Vert_{L^q\left(\Omega\right)}\\
        & \leq C_2\left(\left\|V-Du_n\right\|_{L^p\left(\Omega\right)}+\left\|u-u_n\right\|_{L^p\left(\Omega\right)}\right)\\
        & < 2C_2\varepsilon,
    \end{align}
    where $q=\frac{p}{p-1}\in\left[1,\infty\right]$ is the H\"older conjugate of $p$ and $C_2=C_2\left(p,\phi\right)$ is a constant.

    We use $\left\Vert u-u_n\right\Vert_{L^p\left(\Omega\right)}=\lim_m\left\Vert u_m-u_n\right\Vert_{L^p\left(\Omega\right)} \leq \varepsilon$ at the last inequality.
    As a result, $u_n$ converges to $u\in W^{1,p}\left(\Omega\right)$ with $Du=V$ in $W^{1,p}\left(\Omega\right)$-norm.
\end{proof}

The above theorem shows that for sequences $u_n$ and $V_n$ that minimize VS-PINN loss $\cL^{VS}_p$, $u_n$ converges to $u\in W^{1,p}\left(\Omega,\bR\right)$ provided $V_n$ converges to $V\in L^p\left(\Omega,\bR^d\right)$. 
This means that as long as the VS-PINN loss converges to zero and the auxiliary variable converges, the gradient matching loss term mitigates the derivative pathology that impedes convergence of PINNs, and leads to the convergence of the primary variable from Lemma \ref{lem:Poincare_bdy}.

Although the theorem above provides the convergence of $u_n$ and $V_n$, from two loss terms $\cL_{D,p}$ and $\cL_{\cB,p}$, it does not specify whether they converge to the desired solution and its gradient, respectively.
However, the PDE loss leads them to converge to a generalized solution and its gradient.
The following theorem demonstrates the convergence of VS-PINNs to a generalized solution.

\begin{theorem}\label{thm:conv_weak_2ndlinear}
    Suppose the PDE \eqref{eq:pde} satisfies Assumptions \ref{assump:domain}---\ref{assump:sol}.
    For a given $1\le p\le \infty$, let $u_n\in W^{1,p}\left(\Omega,\bR\right)\cap C\left(\bar{\Omega},\bR\right)$ and $V_n\in W^{1,p}\left(\Omega,\bR^d\right)$ be sequences such that
    \begin{equation}\label{eq:thm_conv_weak_2ndlinear_splitloss_cond}
        \cL^{VS}_p\left(u_n,V_n\right)\rightarrow 0 \text{ as } n\rightarrow\infty.
    \end{equation}
    If $V_n$ converges in $L^p\left(\Omega\right)$ then $u_n$ converges to a generalized solution of \eqref{eq:pde} in $W^{1,p}\left(\Omega\right)$.
\end{theorem}

\begin{proof}
    By Theorem \ref{thm:u_convergence_from_V}, $u_n$ converges to $u\in W^{1,p}\left(\Omega,\bR\right)$ in $W^{1,p}\left(\Omega\right)$, i.e.,
    \begin{equation}\label{eq:thm_conv_weak_2ndlinear_pf_uk_conv}
        \left\Vert u_n-u\right\Vert_{L^p\left(\Omega\right)} +\left\Vert Du_n-Du\right\Vert_{L^p\left(\Omega\right)}\rightarrow 0,
    \end{equation}
    as $n\rightarrow \infty$.
    Then, we have the boundary condition of $u$ from the continuity of the trace operator $T$;
    \begin{equation}
        g=\underset{n\rightarrow\infty}{\lim}u_n\vert_{\partial\Omega} = \underset{n\rightarrow\infty}{\lim} \left(Tu_n\right) = T \left(\underset{n\rightarrow\infty}{\lim} u_n\right) = Tu\quad \text{in }L^p\left(\partial\Omega\right).
    \end{equation}
    
    To show that $u$ satisfies the condition \eqref{eq:weak_sol}, let $\varepsilon >0$ and $\phi\in C_c^\infty\left(\Omega\right)$ be a test function.
    From the following identities
    \begin{align}
        \int_\Omega a_{ij}D_iuD_j\phi\diff x &= \int_\Omega a_{ij}\left(D_iu-D_iu_n\right)D_j\phi\diff x\\
        & \quad+ \int_\Omega a_{ij}\left(D_iu_n-\left(V_n\right)_i\right)D_j\phi\diff x
        + \int_\Omega a_{ij}\left(V_n\right)_iD_j\phi\diff x,
    \end{align}
    and 
    \begin{equation}
        \int_\Omega \phi b_iD_iu\diff x= \int_\Omega \phi b_i\left(D_iu - D_iu_n\right)\diff x
        + \int_\Omega\phi b_i\left(D_iu_n - \left(V_n\right)_i\right)\diff x
        + \int_\Omega \phi b_i \left(V_n\right)_i\diff x,
    \end{equation}
    we have
    \begin{align}\label{eq:thm1_3}
        \cJ\left(u,\phi\right)-\int_\Omega f\phi\diff x&= \int_\Omega \left(a_{ij}D_iuD_j\phi+\phi b_iD_iu+cu\phi-f\phi\right)\diff x \\
        &= \underbrace{\left[\int_\Omega a_{ij}\left(D_iu-D_iu_n\right)D_j\phi\diff x + \int_\Omega \phi b_i\left(D_iu-D_iu_n\right)\diff x \label{eq:thm1_1}\right]}_{\eqqcolon \cI_1}\\
        &\quad+ \underbrace{\left[\int_\Omega a_{ij}\left(D_iu_n-\left(V_n\right)_i\right)D_j\phi\diff x + \int_\Omega \phi b_i\left(D_iu_n-\left(V_n\right)_i\right)\diff x \label{eq:thm1_2}\right]}_{\eqqcolon \cI_2} \\
        &\quad+ \underbrace{\left[\int_\Omega \left(a_{ij}\left(V_n\right)_iD_j\phi+\phi b_i\left(V_n\right)_i+cu_n\phi-f\phi\right)\diff x\right]}_{\eqqcolon \cI_3}\\
        &\quad+ \underbrace{\left[\int_\Omega c\left(u-u_n\right)\phi \diff x\right]}_{\eqqcolon \cI_4}.
    \end{align}
    Both terms $\cI_1$ and $\cI_4$ are bounded by the convergence of $u$, because
    \begin{align}
        \left\vert\cI_1\right\vert & \leq \left\Vert A\right\Vert_{L^\infty\left(\Omega\right)}\left\vert\int_\Omega \left(D_iu-D_iu_n\right)D_j\phi\diff x\right\vert + \left\Vert b\right\Vert_{L^\infty\left(\Omega\right)}\left\vert\int_\Omega \phi \left(D_iu-D_iu_n\right)\diff x\right\vert\\
        & \leq \left(d\left\Vert A\right\Vert_{L^\infty\left(\Omega\right)}\left\Vert D\phi\right\Vert_{L^q\left(\Omega\right)} + \left\Vert b\right\Vert_{L^\infty\left(\Omega\right)}\left\Vert \phi\right\Vert_{L^q\left(\Omega\right)}\right)\left\Vert D_iu-D_iu_n\right\Vert_{L^p\left(\Omega\right)}\\
        &\leq C_1\left\Vert Du-Du_n\right\Vert_{L^p\left(\Omega\right)},
    \end{align}
    and
    \begin{equation}
        \left\vert \cI_4\right\vert = \left\vert\int_\Omega c\left(u-u_n\right)\phi \diff x\right\vert \leq \left\Vert c\right\Vert_{L^\infty\left(\Omega\right)}\left\Vert u-u_n\right\Vert_{L^p\left(\Omega\right)}\left\Vert\phi\right\Vert_{L^q\left(\Omega\right)} = C_2\left\Vert u-u_n\right\Vert_{L^p\left(\Omega\right)},
    \end{equation}
    where $C_1= C_1\left(d, A, b, \phi\right)$ and $C_2= C_2\left(c,\phi\right)$ are constants independent of $n$:
    \begin{align}
        C_1 & = d^2\left\Vert A\right\Vert_{L^\infty\left(\Omega\right)}\left\Vert D\phi\right\Vert_{L^q\left(\Omega\right)} + d\left\Vert b\right\Vert_{L^\infty\left(\Omega\right)}\left\Vert\phi\right\Vert_{L^q\left(\Omega\right)},\\
        C_2 & = \left\Vert c\right\Vert_{L^\infty\left(\Omega\right)}\left\Vert\phi\right\Vert_{L^q\left(\Omega\right)}.
    \end{align}
    Therefore,
    \begin{equation}
        \left\vert\cI_1+\cI_4\right\vert\leq \max\left\{C_1,C_2\right\}\left\Vert u-u_n\right\Vert_{W^{1,p}\left(\Omega\right)} \rightarrow0, \text{ as } n\rightarrow\infty.
    \end{equation}
    Also, remained terms $\cI_2$ and $\cI_3$ are bounded via loss terms $\cL_{\cN,p}$ and $\cL_{D,p}$;
    \begin{align}
        \left\vert\cI_2\right\vert & \leq d^2\left\Vert A \right\Vert_{L^\infty\left(\Omega\right)}\left\Vert Du_n-V_n \right\Vert_{L^p\left(\Omega\right)}\left\Vert D\phi\right\Vert_{L^{q}\left(\Omega\right)}\\
        & \qquad +d\left\Vert b \right\Vert_{L^\infty\left(\Omega\right)}\left\Vert Du_n-V_n \right\Vert_{L^p\left(\Omega\right)}\left\Vert \phi\right\Vert_{L^{q}\left(\Omega\right)}
        \\
        & \leq \left(d^2\left\Vert A \right\Vert_{L^\infty\left(\Omega\right)}\left\Vert D\phi\right\Vert_{L^{q}\left(\Omega\right)} + d\left\Vert b \right\Vert_{L^\infty\left(\Omega\right)}\left\Vert \phi\right\Vert_{L^{q}\left(\Omega\right)}\right)\left\Vert Du_n-V_n \right\Vert_{L^p\left(\Omega\right)}\\
        & = C_1\cL_{D,p}\left(u_n,V_n\right)\\
        & \leq \lambda_{\cD}^{-1}C_1 \cL^{VS}_p\left(u_n,V_n\right),\\
        \left\vert\cI_3\right\vert & = \left\vert \int_\Omega \left(-D_j\left(a_{ij}\left(V_n\right)_i\right)+b_i\left(V_n\right)_i +cu_n-f\right)\phi\diff x\right\vert \\
        & \leq \left\Vert-D_j\left(a_{ij}\left(V_n\right)_i\right)+b_i\left(V_n\right)_i +cu_n-f \right\Vert_{L^p\left(\Omega\right)} \cdot \left\Vert \phi \right\Vert_{L^q\left(\Omega\right)}\\
        & \leq C_3 \cL_{\cN,p}\left(u_n, V_n\right)\\
        & \leq \lambda_{\cN}^{-1}C_3\cL^{VS}_p\left(u_n,V_n\right),
    \end{align}
    where $C_3=C_3\left(\phi\right)=\left\Vert\phi\right\Vert_{L^q\left(\Omega\right)}$.
    Thus, the convergence of the loss $\cL^{VS}_p$ induces
    \begin{equation}
        \left\vert \cI_2+\cI_3\right\vert \leq \left(\lambda_D^{-1}C_1+\lambda_{\cN}^{-1}C_3\right) \cL^{VS}_p\left(u_n, V_n\right) \rightarrow 0, \text{ as } n\rightarrow\infty.
    \end{equation}
    
    Combining all, we arrive at 
    \begin{align}
        &\left\vert \int_\Omega \left(a_{ij}D_iuD_j\phi+\phi b_iD_iu+cu\phi-f\phi\right)\diff x \right\vert\\
        &\leq\left\vert\cI_1 + \cI_4\right\vert + \left\vert\cI_2 + \cI_3\right\vert\\
        &\leq \max\left\{C_1,C_2\right\}\left\Vert u-u_n\right\Vert_{W^{1,p}\left(\Omega\right)}+\left(\lambda_D^{-1}C_1+\lambda_{\cN}^{-1}C_3\right) \cL^{VS}_p\left(u_n, V_n\right)\\
        & < \varepsilon,
    \end{align} 
    for sufficiently large $n$.
    We obtain the desired result as $\varepsilon>0$ and $\phi\in C_c^\infty\left(\Omega\right)$ were arbitrary.
\end{proof}
Theorem \ref{thm:conv_weak_2ndlinear} states that VS-PINNs can achieve a generalized solution if Criteria \ref{crit:loss} and \ref{crit:network} are satisfied.
It is worth noting that Criterion \ref{crit:network} only requires considering the convergence of one variable $V$, even though VS-PINNs have two variables $u$ and $V$.

Note that Theorem \ref{thm:u_convergence_from_V} utilizes the gradient matching loss $\cL_{D,p}$ in \eqref{eq:vs_loss_gm} and the boundary loss $\cL_{\cB,p}$ in \eqref{eq:pinn_loss_bdy}, as well as Criterion \ref{crit:network}, which is exclusively applied to auxiliary variables.
Theorem \ref{thm:conv_weak_2ndlinear} employs the PDE loss $\cL_{\cN,p}$ in \eqref{eq:vs_loss_pde} and boundary loss $\cL_{\cB,p}$ to demonstrate that the primary variable converges to a generalized solution.
The use of distinct loss terms in each theorem clarifies their role and emphasizes their importance in attaining the solution.

\subsection{Discussion and Implications}
We analyzed the failure of PINNs in Section \ref{sec:pinn_failure} and the convergence of the proposed VS-PINNs in this section.
In addition to convergence, VS-PINNs benefit from either parameterizing the gradient or lowering the derivative order in PDEs.
Moreover, VS-PINNs formulate the loss directly from the PDE without having to derive an integral form to approximate the generalized solution and do not require a hand-crafted approach for solving different PDEs.
We discuss the advantages and implications of VS-PINNs.

\begin{itemize}
\item \textbf{Equivalence of the problem.}
We first discuss the equivalence between VS-PINNs and PINNs in terms of the optimization problem.
PINNs and VS-PINNs correspond to the soft penalization of the equivalent PDEs \eqref{eq:pde} and \eqref{eq:pde_split}, respectively, indicating that both methods originate from the same problem.
One might consider the variable splitting strategy as a constrained optimization version of the residual minimization of PINNs \eqref{eq:pinn} by employing an auxiliary variable as follows:
\begin{align}\label{eq:split_constrained}
    \begin{split}
    \underset{u,V}{\text{minimize  }}\ & \lambda_{\cN}\left\Vert\cN \left[u,V\right]-f\right\Vert_{L^p\left(\Omega\right)} + \lambda_\cB\left\Vert u-g \right\Vert_{L^p\left(\partial\Omega\right)}, \\
    \text{subject to } &\ D u = V,
    \end{split}
\end{align}
which remains an equivalent problem to PINNs \eqref{eq:pinn}.

\item \textbf{Convergence from independently parameterizing the gradient.}
Although PINNs and VS-PINNs solve an equivalent problem, their difference in convergence is noticeable, as demonstrated in Sections \ref{sec:pinn_failure} and \ref{sec:variable_splitting}.
PINNs cannot guarantee convergence to the desired PDE solution even if the loss approaches zero (Criterion \ref{crit:loss}) and the network output converges (Criterion \ref{crit:network}), as discussed in Section \ref{sec:pinn_failure}.
In contrast, VS-PINNs guarantee the convergence of the primary variable to a generalized solution of the governing PDE if the loss approaches zero and the output of the network converges.
Indeed, the proposed method requires an even weaker assumption than Criterion \ref{crit:network}, as Theorem \ref{thm:conv_weak_2ndlinear} requires only the convergence of the auxiliary variables.
The most significant factor is the convergence of the gradient; hence, if PINNs can ensure the convergence of both the network and its gradient, we can consider the convergence to a generalized solution.
However, evaluating or quantifying a network's gradient convergence is challenging, which is beyond the scope of Criteria \ref{crit:loss} and \ref{crit:network}.

\item \textbf{Absence of higher-order derivatives in training.}
As previously mentioned in Section \ref{sec:pinn_failure}, it has been highlighted that PINN loss with high-order partial derivatives poses a challenge for gradient-based optimizers when minimizing the loss.
In addition, when handling high-order PDEs in high-dimensional domains, PINNs encounter an increasing computational cost as the order of the PDE increases, resulting in substantial inefficiency from stacked back-propagation.
As a result, PINNs become impractical for large-scale settings.
By contrast, variable splitting excludes higher-order derivatives by expressing the second-order PDE \eqref{eq:pde} solely in terms of first-order derivatives, as in \eqref{eq:pde_split}.
Moreover, the variable splitting strategy is not restricted to second-order PDEs, as noted in Remark \ref{rem:split_higher_order}.
We expect that VS-PINNs would be more efficient than PINNs as the order of the PDEs increases; this remains to be proven in future research.

\item \textbf{Mitigation of unnecessary higher-order regularities.}
Variable splitting also helps to alleviate the regularity issue raised for PINNs when approximating generalized solutions rather than classical solutions.
PINNs require that the decision variable has the same degree of regularity as the order of the governing PDE.
This poses the concern that they may have struggled to approximate generalized solutions that are not necessarily differentiable everywhere by that order \citep{shin2020convergence,zang2020weak}.
For second-order linear PDEs \eqref{eq:pde}, VS-PINNs require first-order derivatives, while PINNs necessitate computation of a second-order derivative of the predicted solution.

\item \textbf{Wider searching space.}
A decrease in the derivative order offers the benefit of approximating the solution within a wider function space.
PINNs require the searching space of the decision variable $u$ to be $W^{2,p}\left(\Omega,\bR\right)$, which may exclude a true generalized solution $u^\ast\in W^{1,p}\left(\Omega,\bR\right)$.
Alternatively, VS-PINNs can exploit a sufficiently wide function space that includes the target generalized solution.
Using the auxiliary variable $V\in W^{1,p}\left(\Omega,\bR^d\right)$, VS-PINNs expand the searching space for the primary variable $u$ to $W^{1,p}\left(\Omega,\bR\right)$, where a generalized solution is defined.

This implies that VS-PINNs can use a broader range of neural networks to represent the primary variable $u$.
In feed-forward neural networks, activation functions determine the regularity of the networks.
Thus, it is possible to utilize sharp activation functions such as rectified linear units (ReLU) and leaky-ReLU, which are not feasible for PINNs.
Using deep networks, \cite{park2020minimum} demonstrated that ReLU networks achieve a critical width that allows a neural network to be a universal approximator of $L^p$ functions.
From this perspective, adopting ReLU-type networks is another strength of variable splitting.

\item\textbf{Direct utilization of PDEs rather than integral formulation.}
VS-PINNs yield a generalized solution by minimizing a loss function that directly incorporates the PDE without requiring an integral formulation \eqref{eq:weak_sol}.
Several methods have been proposed \citep{zang2020weak, de2022wpinns, chen2023friedrichs} to approximate generalized solution; however, most of these methods directly employ the integral formulation \eqref{eq:weak_sol} as the loss function and parameterize the test function as another network.
Nevertheless, these approaches typically result in a min-max problem, where the training is known to be significantly unstable and the network does not represent the entire test function space.
On the other hand, VS-PINNs are free from these issues because they solve the residual minimization of a given PDE without using an integral formulation.

\end{itemize}

\section{Conclusion}\label{sec:conclusion}
Recently, using deep learning techniques for the numerical solutions of PDEs has rapidly grown alongside their advances, with PINNs being one of the most prominent approaches.
Despite the recent surge in studies on PINNs across various fields, they often struggle to accurately approximate the solution of the governing PDE. 

We showed the fundamental problem with PINNs that convergence to a PDE solution cannot be guaranteed, even if the loss is sufficiently reduced to zero. 
This indicates that the premise underlying the residual minimization of PINNs is invalid. 
Previous studies on the failure of PINNs have often assumed this premise to be true; however, its validity has not been thoroughly analyzed theoretically.  
Given that the primary goal of numerical methods is to approximate analytically intractable solutions of PDEs for many applications, non-convergence precludes the use of PINNs in real-world scenarios. 
Moreover, the fundamental problem with PINNs stems from their inability to control the gradient of the predicted solutions. 
The derivative pathology of PINNs, which has been previously overlooked, highlights the importance of considering derivatives to approximate the solution of PDEs using PINNs. 

Our new theoretical understanding of the failure modes of PINNs allowed us to address the derivative pathology by proposing a variable splitting strategy that parameterizes the gradient of the solution as an auxiliary variable. 
In our study, we showed that the variable splitting, by incorporating an auxiliary variable, effectively addresses the issue of derivative pathology.
This approach successfully ensures convergence to a generalized solution of second-order linear PDEs.
Convergence was proven for the generalized solution, which is a broader concept than the classical solution, making it applicable to various problems. 
In addition to convergence, variable splitting offers various benefits, such as mitigating the inefficiencies of PINNs by reducing the order of derivatives in the loss function. 
In short, this study takes a step forward in developing a reliable methodology for deep-learning-based techniques for solving PDEs, where existing methods lack theoretical understanding.

Although we presented a variable splitting strategy with provable convergence guarantees, we are still in the early stages of building reliable methods, which raises numerous questions regarding both VS-PINNs and PINNs.

In practice, a discretized empirical loss is used instead of a continuous loss. 
In PINNs, the generalization error is considered in terms of the empirical loss and the distribution of collocation points, assuming that the premise holds and that the minimizer of the continuous loss becomes a solution to the governing PDE.
Because VS-PINNs guarantee this premise, the generalization error is a concern for practical use.

There is also a question regarding the efficiency of the optimization:
Does variable splitting alleviate the optimization problem of PINNs by representing a PDE with lower-order derivatives? Based on optimization literature, advanced optimization schemes such as the augmented Lagrangian method can be applied to VS-PINNs.
Studies on the convergence rate of VS-PINNs and whether advanced schemes accelerate convergence are new research topics. 
On the other hand, the use of an additional variable and gradient matching loss term in VS-PINNs may hinder optimization in practice.
This is because it involves optimizing over two independent variables, which may lead to an increase in spurious local minima. Specifically, a stable solution could potentially exist where no variable can improve its situation by unilaterally changing its state.
Also, the inclusion of the gradient matching loss may exacerbate discrepancies between loss terms.
These pose another intriguing future work.

However, there is scope for improving the theoretical results of this study.
The development of variable splitting is not confined to second-order linear PDEs but can be straightforwardly generalized to other higher-order or nonlinear PDEs.
However, convergence to other PDEs has yet to be proven and remains a topic for future research. In addition, we assumed convergence of the auxiliary variable.
We expect to dispense with assumption on an auxiliary variable, with additional terms such as various regularization terms.
In addition to the VS-PINNs, the convergence analysis of PINNs with regularization, beyond the most basic PINN loss we covered, would also be an interesting research direction.

The feasibility of Criterion \ref{crit:loss} also needs to be considered for PDEs with generalized solutions.
While a generalized solution exists in the searching space for the primary variable of VS-PINNs, it is still unclear whether every generalized solution minimizes the proposed VS-PINN loss to zero.
Consequently, another crucial aspect that remains open for both PINNs and VS-PINNs is the existence of networks with sufficiently small loss.
Note that
it is feasible for classical solutions of PDEs.
Specifically, a classical solution minimizes the loss to zero, and hence its approximation has a small loss.
Therefore, Criterion \ref{crit:loss} can be satisfied for any PDEs with a classical solution. 
We would like to emphasize that even in this case, only VS-PINNs guarantee convergence.
As demonstrated in Sections \ref{sec:pinn_failure} and \ref{sec:convergence_VS}, PINNs may fail even when the decision variable satisfies both criteria, but VS-PINNs converge to a generalized solution.

We believe that answering these questions not only helps to better understand the proposed method, but also paves a new way for developing reliable scientific machine learning models, as needed for many critical applications in science and engineering.


\acks{We would like to acknowledge support for this study
from the NRF grant [2021R1A2C3010887]}.

\vskip 0.2in
\bibliography{mybib}

\end{document}